%% file: main.tex
\pdfoutput=1

\documentclass[letterpaper]{article} 
\usepackage{geometry}
\savegeometry{articlegeometry}
\usepackage{aaai23}  
\usepackage{times}  
\usepackage{helvet}  
\usepackage{courier}  
\usepackage[hyphens]{url}  
\usepackage{graphicx} 
\usepackage{natbib}  
\usepackage{caption} 
\usepackage[algoruled, lined, shortend, linesnumbered]{algorithm2e}
\DontPrintSemicolon
\SetDataSty{textit}
\SetFuncSty{textsc}
\SetProcNameSty{textsc}
\SetProcArgSty{textit}

\SetCommentSty{mycommentstyle}
\SetKwBlock{Repeat}{repeat}{}
\SetKwInput{KArg}{Argument}
\SetKwInput{KRet}{Returns}
\SetKwFunction{FMakeShaping}{Shaping}
\SetKwFunction{FStopCondition}{Stop}
\SetKwFunction{FAction}{Action}
\SetKwFunction{FAct}{Act}
\SetKwFunction{FState}{State}
\SetKwFunction{FUpdate}{Update}
\SetKwFunction{FOutput}{Output}
\SetKwFunction{FComputeValue}{ComputeValue}
\SetKwData{DLearner}{Learner}

\usepackage{comment}
\usepackage{amsfonts}
\usepackage{amsmath}
\usepackage{paralist}
\usepackage{subfig}

\usepackage{amsthm}
\theoremstyle{plain}
\newtheorem{theorem}{Theorem}
\newtheorem{lemma}{Lemma}
\newtheorem{proposition}{Proposition}

\theoremstyle{definition}
\newtheorem{definition}{Definition}

\newtheorem{assumption}{Assumption}
\theoremstyle{remark}

\usepackage{mathtools}

\PassOptionsToPackage{svgnames,cmyk}{xcolor}
\usepackage{tikz}
\usetikzlibrary{arrows}
\usetikzlibrary{arrows.spaced}
\usetikzlibrary{automata}
\usetikzlibrary{fit}
\usetikzlibrary{positioning}
\usetikzlibrary{calc}
\usetikzlibrary{graphs}
\usetikzlibrary{graphs.standard}
\usetikzlibrary{matrix}
\usetikzlibrary{intersections}
\usetikzlibrary{shapes.geometric}
\usetikzlibrary{decorations.pathreplacing}
\usetikzlibrary{patterns}
\usetikzlibrary{quotes}
\usetikzlibrary{backgrounds}

\pgfdeclarelayer{below}
\pgfdeclarelayer{beloww}
\pgfsetlayers{beloww,below,main}

\tikzset{
	>=stealth',
	ampersand replacement=\&,
	font=\footnotesize,
	every node/.style={align=center},
	node/.style={shape=circle, thick, draw, minimum size=2.5mm, inner sep=0},
	observed node/.style={node, fill=lightgray},
	block/.style={
		rounded corners=4pt, draw, fill=white, inner sep=6pt, outer sep=1pt},
	box/.style={rounded corners=4pt, inner sep=1.8ex, draw=#1!50,
		fill=#1!20},
	dot/.style={fill, circle, inner sep=0.7pt},
	image/.style={inner sep=0pt, outer sep=3pt},
	tight/.style={inner sep=0pt, outer sep=0pt},
	state/.style={state without output, fill=white, inner sep=0pt, minimum size=8pt},
	flow/.style={->, thick, >=spaced stealth'},
	arrow/.style={->, semithick},
	pin line/.style={-, very thin, draw=gray},
	label text/.style={draw=none, font=\scriptsize, inner sep=1pt},
	shorten <>/.style={ shorten >=#1, shorten <=#1 },
}

\colorlet{noshcolor}{orange!50!black}
\colorlet{delayedqcolor}{teal}
\colorlet{shcolor}{green!50!black}
\colorlet{shinvcolor}{blue!50!lightgray}
\colorlet{bad1color}{teal}
\colorlet{bad2color}{purple}

\tikzset{
	our-rs/.style={draw=shcolor},
	inv-rs/.style={draw=shinvcolor, densely dashed},
	no-rs/.style={draw=noshcolor, densely dashdotted},
	delayedq-rs/.style={draw=delayedqcolor, dashdotted},
	bad1-rs/.style={draw=bad1color, densely dash dot dot},
	bad2-rs/.style={draw=bad2color, densely dotted},
}

\definecolor{grid-r}{RGB}{255,0,0}
\definecolor{grid-g}{RGB}{0,255,0}
\definecolor{grid-b}{RGB}{180,180,232}
\definecolor{grid-y}{RGB}{255,255,0}
\definecolor{grid-p}{RGB}{255,192,203}
\definecolor{grid-o}{RGB}{255,165,0}
\definecolor{grid-B}{RGB}{214,98,98}
\definecolor{grid-G}{RGB}{190,190,190}

\DeclarePairedDelimiter{\abs}{\lvert}{\rvert}

\newcommand{\Reals}{\mathbb{R}}
\newcommand{\Naturals}{\mathbb{N}}
\newcommand{\Expected}{\mathbb{E}}
\newcommand{\Given}{\mid}

\newcommand{\Indicator}{\mathbb{I}}

\renewcommand{\implies}{\Rightarrow}
\newcommand{\Const}[1]{\mathit{#1}}

\renewcommand{\implies}{\Rightarrow}

\newcommand{\SetSym}[1]{\mathcal{#1}}
\newcommand{\Distribution}[1]{\Delta(#1)}

\newcommand{\Trace}{\pi}
\newcommand{\Policy}{\rho}
\newcommand{\PolicySpace}{P}  

\newcommand{\States}{\SetSym{S}}
\newcommand{\Actions}{\SetSym{A}}
\newcommand{\Model}{\SetSym{M}}
\newcommand{\Rewards}{\SetSym{R}}
\newcommand{\Mapping}{\phi}
\newcommand{\Potential}{\Phi}

\newcommand{\Automa}{\mathcal{D}}

\newcommand{\Shaped}[1]{#1^{s}}
\newcommand{\Biased}[1]{#1^{b}}
\newcommand{\BiasedStar}[1]{#1^{b*}}
\newcommand{\Est}[1]{\hat{#1}}
\newcommand{\Abst}[1]{\bar{#1}}
\newcommand{\Goals}{\SetSym{G}}

\newcommand{\Options}{\SetSym{O}}
\newcommand{\ReturnSym}{G}

\newcommand{\FromBlockTo}[3]{{#3}, {#2}}

\newcommand{\ExplorationPolicy}{\Policy_e}
\newcommand{\RlAlgo}{\mathcal{L}}
\newcommand{\POptAccurate}{\epsilon}
\newcommand{\PHomogeneity}{\nu}

\newenvironment{theoremmain}[1]{\renewcommand\thetheorem{#1}\theorem}{\endtheorem}
\newenvironment{lemmamain}[1]{\renewcommand\thelemma{#1}\lemma}{\endlemma}
\newenvironment{propositionmain}[1]{\renewcommand\theproposition{#1}\proposition}{\endproposition}

\newcommand{\refmain}[1]{\ref{#1}}
\newcommand{\eqrefmain}[1]{\eqref{#1}}
\newcommand{\refthis}[1]{\ref{#1}}
\newcommand{\eqrefthis}[1]{(\ref{#1})}

\title{Exploiting Multiple Abstractions in Episodic RL via Reward Shaping}
\author{
	Roberto Cipollone\textsuperscript{\rm 1},
	Giuseppe De Giacomo\textsuperscript{\rm 1,\rm 2},
	Marco Favorito\textsuperscript{\rm 3},
	Luca Iocchi\textsuperscript{\rm 1},
	Fabio Patrizi\textsuperscript{\rm 1}
}
\affiliations{
	\textsuperscript{\rm 1}DIAG, Università degli Studi di Roma ``La Sapienza'', Italy\\
	\textsuperscript{\rm 2}Department of Computer Science, University of Oxford, U.K.\\
	\textsuperscript{\rm 3}Banca d'Italia, Italy\\
	\{cipollone, degiacomo, iocchi, patrizi\}@diag.uniroma1.it\\
	giuseppe.degiacomo@cs.ox.ac.uk,\, marco.favorito@bancaditalia.it
}

\begin{document}

\maketitle

\begin{abstract}
	One major limitation to the applicability of Reinforcement Learning (RL)
	to many practical domains is the large number of samples
	required to learn an optimal policy.
	To address this problem and improve learning efficiency, 
	we consider a linear hierarchy of abstraction layers
	of the Markov Decision Process (MDP) underlying the target domain.
	Each layer is an MDP representing a coarser model of the one immediately below in the hierarchy.
	In this work, we propose a novel form of Reward Shaping where 
	the solution obtained at the abstract level is used to offer
	rewards to the more concrete MDP, in such a way that
	the abstract solution guides the learning in the more complex domain.
	In contrast with other works in Hierarchical RL, our technique has few requirements in the 
	design of the abstract models and it is also tolerant to modeling errors, thus making the
	proposed approach practical.
	We formally analyze the relationship between the abstract models and 
	the exploration heuristic induced in the lower-level domain.
	Moreover, we prove that the method guarantees optimal convergence
	and we demonstrate its effectiveness experimentally.
\end{abstract}

\section{Introduction}
In Reinforcement Learning (RL), agents have no complete 
model available to predict the outcomes of their actions. 
Since coming up with a complete and faithful model of the world is generally difficult,
this allows for the wide applicability of RL algorithms.
Nonetheless, such a lack of knowledge also demands a significant number of interactions with the environment,
before a (near-)optimal policy can be estimated.
As a result, most of the successes of RL come from the digital world (e.g., video games, simulated environments),
especially those in which a large number of samples can be easily generated,
while applications to real environments, such as robotic scenarios, 
are still rare.

Many RL tasks are \emph{goal-oriented},
which implies that when the goal states are \emph{sparse}, so are the rewards.
This is a well-known challenging scenario for RL, which further increases the amount of samples required.
Unfortunately, sparse goal states are very common,
as they may arise in simple tasks,
such as reaching specific configurations in large state spaces,
or, even in modest environments, for complex target behaviours, such as the successful completion of a sequence of smaller tasks \cite{brafman_2018_LTLfLDLf,icarte2018using}.

In order to improve sample efficiency, \emph{Hierarchical} RL approaches \cite{hutsebaut-buysse_2022_HierarchicalReinforcementb} have been proposed to decompose 
a complex problem into subtasks that are easier to solve individually.
In this context, an \emph{abstraction} is a simplified, coarser model that reflects the 
decomposition induced over the \emph{ground}, more complex environment \cite{li06towards}.

In this paper, we consider a linear hierarchy of abstractions of the Markov Decison Process (MDP) underlying the target domain.
Each layer in this hierarchy is a simplified model, still represented as an MDP, of the one immediately below.
A simple example is that of an agent moving in a map.
The states of the ground MDP may capture the real pose of the agent, as continuous coordinates and orientation.
The states of its abstraction, instead, may provide a coarser description, obtained by coordinate discretization, semantic map labelling (i.e., associating metric poses to semantic labels), or by projecting out state variables.
Ultimately, such a compression corresponds to partitioning the ground state space into abstract states, implicitly defining a mapping from the former to the latter.
The action spaces of the two models can also differ, in general, as they would include the actions that are best appropriate for each representation.
%
Simulators are commonly used in RL and robotics.
Often, simply though a different configuration of the same software, e.g. noiseless or ideal movement actions,
it is possible to obtain a simplified environment which acts as an abstraction.
Importantly, this simplified model also applies to a variety of tasks that may be defined over the same domain.

By taking advantage of the abstraction hierarchy, we devise an approach which allows any off-policy RL algorithm to efficiently explore the ground environment, while guaranteeing optimal convergence.
The core intuition is that the value function $\Abst{V}^*$ of the abstract MDP~$\Abst\Model$ can be exploited to guide learning on the lower model~$\Model$.
Technically, we adopt a variant of Reward Shaping (RS), whose potential is generated from~$\Abst{V}^*$.
This way, when learning in~$\Model$, the agent is \emph{biased} 
to visit first the states that correspond in~$\Abst\Model$ to those that are preferred by the abstract policy,
thus trying, in a sense, to replicate its behavior in~$\Model$.
%
In order to guarantee effectiveness of the exploration bias, it is essential that the transitions of~$\Abst\Model$ are good proxies for the dynamics of~$\Model$.
We characterize this relationship by identifying conditions under which the abstraction induces an exploration policy that is consistent with the ground domain.

The contributions of this work include:
\begin{inparaenum}[(i)]
    \item the definition of a novel RS schema that allows for transferring the acquired experience from coarser
        models to the more concrete domains in the abstraction hierarchy;
    \item the derivation of a relationship between each abstraction and the exploration policy that is induced in the lower MDP;
    \item the identification of \emph{abstract value approximation}, as a new condition to evaluate when an abstraction can be a good representative of ground MDP values;
	\item an experimental analysis showing that our approach significantly improves sample-efficiency and that modelling errors yield only a limited performance degradation.
\end{inparaenum}

\section{Preliminaries}

\paragraph{Notation}
Any total function $f: \SetSym{X} \to \SetSym{Y}$ induces a partition on its
domain~$\SetSym{X}$, such that two elements are in the same block iff $f(x) = f(x')$. 
We denote blocks of the partition by the elements of $\SetSym{Y}$, thus writing $x \in y$
instead of $x \in \{x' \Given {x' \in \SetSym{X}}, f(x') = y\}$.
With $\Distribution{\SetSym{Y}}$, we denote the class of probability distributions
over a set~$\SetSym{Y}$.
Also, ${f: \SetSym{X} \to \Distribution{\SetSym{Y}}}$ is a function returning a probability distribution (i.e., $f: \SetSym{X},
\SetSym{Y} \to [0, 1]$, with $\sum_{y \in \SetSym{Y}} f(x, y) = 1$, for
all $x \in \SetSym{X}$).

\paragraph{Markov Decision Processes (MDPs)}
A Markov Decision Process (MDP) $\Model$ is a tuple $\langle \States, \Actions, T, R, \gamma
\rangle$, where $\States$ is a set of states, $\Actions$ is a set of actions, $T:
\States \times \Actions \to \Distribution{\States}$ is the probabilistic transition
function, $R: \States \times \Actions \times \States \to \SetSym{R}$ is
the reward function (for $\Rewards \coloneqq [r_-, r_+] \subset \Reals$), and $0 < \gamma < 1$ is the discount factor.
A deterministic policy is a function $\Policy: \States \to \Actions$.
We follow the standard definitions for the value of a policy $V^\Policy(s)$, as the expected sum of discounted returns, the value of an action $Q^\Policy(s, a)$, and their respective optima $V^* = \max_{\Policy \in \PolicySpace} V^{\Policy}(s)$  and $Q^*(s, a)$ \cite{puterman1994}.
We may also write $Q(s, \Policy)$ to denote~$V^\Policy(s)$.
Reinforcement Learning (RL) is the task of learning an optimal policy in an MDP with unknown~$T$ and~$R$.

\begin{definition}
    \label{def:off-policy-learning}
    A RL learning algorithm is \emph{off-policy} if, for any MDP $\Model$
    and experience $(\Actions\Rewards\States)^t$ generated by a stochastic exploration policy $\ExplorationPolicy: \States, \Naturals \to \Distribution{\Actions}$ such that $\forall s, t, a: \ExplorationPolicy(s, t, a) > 0$,
    the algorithm converges to the optimal policy of $\Model$, as $t \to \infty$.
\end{definition}

\paragraph{Reward Shaping (RS)} Reward Shaping (RS) is a technique for learning in MDPs with sparse rewards, which occur rarely during exploration.
The purpose of RS is to guide the agent by exploiting some prior knowledge in the form of additional rewards:
$\Shaped{R}(s, a, s') \coloneqq R(s, a, s') + F(s, a, s')$, where $F$ is the \emph{shaping} function. 
In the classic approach, called \emph{Potential-Based RS}~\cite{ng1999policy} (simply called Reward Shaping from now on), the shaping function is defined in terms of a potential function, $\Potential: \States \to \Reals$, as:
\begin{equation}
    F(s, a, s') \coloneqq \gamma\,\Potential(s') - \Potential(s)
	\label{eq:pbrs}
\end{equation}
If an infinite horizon is considered, this definition and its variants \cite{wiewiora2003principled,devlin2012} guarantee that the set of optimal policies of~$\Model$ and $\Shaped\Model \coloneqq \langle \States, \Actions, T, \Shaped{R}, \gamma \rangle$ coincide.
Indeed, as shown by~\cite{wiewiora2003potential}, the Q-learning algorithm over $\Shaped{\Model}$ performs the same updates as Q-learning over $\Model$ with the modified Q-table initialization: $Q'_0(s,a)
\coloneqq {Q_0(s,a) + \Potential(s)}$.

\paragraph{Options}
An option~\cite{sutton1999between}, for an MDP $\Model$, is a temporally-extended action, defined as $o = \langle \SetSym{I}_o, \Policy_o, \beta_o \rangle$,
where $\SetSym{I}_o \subseteq \States$ is an initiation set,
${\Policy_o: \States \to \Actions}$ is the policy to execute,
and ${\beta_o: \States \to \{0, 1\}}$ is a termination condition that, from the current state, computes whether the option should terminate.
We write $Q^*(s, o)$ to denote the expected return of executing~$o$ until termination and following the optimal policy afterwards.

\paragraph{$\Mapping$-Relative Options}
\cite{abel_2020_ValuePreserving}\, Given an MDP~$\Model$ and a function $\Mapping:
\States \to \Abst\States$, an option $o = \langle \SetSym{I}_o, \Policy_o, \beta_o
\rangle$ of~$\Model$ is said to be \emph{$\Mapping$-relative} iff there is
some $\Abst{s} \in \Abst{\States}$ such that:
\begin{equation}
	\SetSym{I}_o = \{s \mid s \in \Abst{s}\},\quad
	\beta_o(s) = \Indicator(s \not\in \Abst{s}),\quad
	\Policy_o \in \PolicySpace_{\Abst{s}}
\end{equation}
where $\Indicator(\xi)$ equals 1 if $\xi$ is true, $0$ otherwise, and
$\PolicySpace_{\Abst{s}}$ is the set of policies of the form
$\Policy_{\Abst{s}}: \{ s \mid s \in \Abst{s} \} \to \Actions$.
In other words, policies of $\Mapping$-relative options are defined within some block~$\Abst{s}$ in the partition induced by~$\Mapping$
and they must terminate exactly when the agent leaves the block where it was initiated.

\section{Framework}

Consider an environment in which experience is costly to obtain.
This might be a complex simulation or an actual environment in which a physical robot is acting.
This is our \emph{ground} MDP~$\Model_0$ that we aim to solve while reducing the number of interactions with the environment.
Instead of learning on this MDP directly, we choose to solve a simplified, related problem, that we call the \emph{abstract} MDP.
This idea is not limited to a single abstraction.
Indeed, we consider a hierarchy of related MDPs $\Model_0, \Model_1, \dots, \Model_n$, of decreasing difficulty, where the experience acquired by an expert acting in $\Model_i$ can be exploited to speed up learning in the previous one,~$\Model_{i-1}$.

Associated to each MDP abstraction~$\Model_i$, we also assume the existence of a \emph{mapping} function ${\Mapping_i: \States_i \to \States_{i+1}}$, which projects states of~$\Model_i$ to states of its direct abstraction~$\Model_{i+1}$.
This induces a partition over $\States_i$, where each block contains
all states that are mapped through~$\Mapping_{i}$ to a single state in~$\Model_{i+1}$.
The existence of state mappings are a classic assumption in Hierarchical RL
\cite{ravindran_model_2002,abel_2016_OptimalBehavior,abel_2020_ValuePreserving},
which are easily obtained together with the abstract MDP.
Unlike other works, instead, we do not require any mapping between the action spaces.
This relationship will remain implicit.
This feature leaves high flexibility to the designers for defining the abstraction hierarchy.

An abstract model is therefore a suitable relaxation of the
environment dynamics. For example, in a navigation scenario, an abstraction could
contain actions that allow to just ``leave the room'', instead of navigating though
space with low-level controls. 
In Section~\ref{sec:quality}, we formalize this intuition by deriving a measure that quantifies how accurate an abstraction is with respect to the lower domain.

As a simple example, consider the grid world domain, the abstract MDP and the mapping in Figure~\ref{fig:rooms8}.
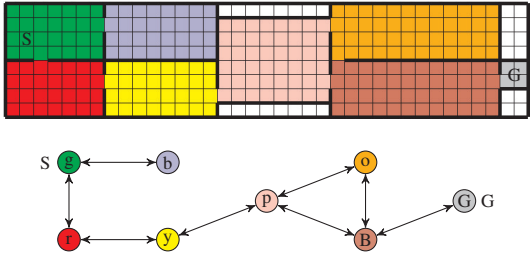
\begin{figure}
	\centering
	\begin{tikzpicture}[every node/.append style={font=\scriptsize}]
	    \node (gridworld) [image] {\includegraphics[width=7cm]{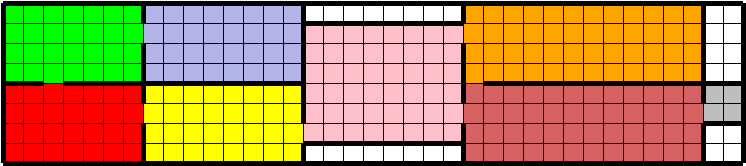}};
	    \begin{scope}[shift=(gridworld.north west), x=(gridworld.north east), y=(gridworld.south west)]
	        \node at (0.058, 0.34) {S};
	        \node at (0.957, 0.61) {G};
	    \end{scope}
	    \matrix [below=0.2cm of gridworld, nodes={state, minimum size=0.3cm}, column sep=1cm, row sep=0.2cm] {
	        \node [fill=grid-g, label=left:S] (g) {g};\&
	        \node [fill=grid-b] (b) {b};\&
	        \&
	        \node [fill=grid-o] (o) {o};\&
	        \\
	        \&
	        \&
	        \node [fill=grid-p] (p) {p};\&
	        \&
	        \node [fill=grid-G, label=right:G] (G) {G};\\
	        \node [fill=grid-r] (r) {r};\&
	        \node [fill=grid-y] (y) {y};\&
	        \&
	        \node [fill=grid-B] (B) {B};\&
	        \\
	    };
	    \draw [<->] (g) -- (b);
	    \draw [<->] (g) -- (r);
	    \draw [<->] (r) -- (y);
	    \draw [<->] (y) -- (p);
	    \draw [<->] (p) -- (o);
	    \draw [<->] (p) -- (B);
	    \draw [<->] (o) -- (B);
	    \draw [<->] (B) -- (G);
	\end{tikzpicture}
	\caption{A grid world domain (top) and an abstraction (bottom). The colors encode the mapping function, G is the goal.}
	\label{fig:rooms8}
\end{figure}
Thanks to the abstraction, we can inform the agent that exploration should avoid the blue ``room'' (b) and only learn options for moving to and within the other rooms.
The same does not necessarily hold for the orange block (o), instead,
as the optimal path depends on the specific transition probabilities in each arc.

\subsection{Exploiting the Knowledge}
Let us consider a hierarchy of abstractions $\Model_0, \dots, \Model_n$, together with
the functions~$\Mapping_{0}, \dots, \Mapping_{n-1}$. The learning process proceeds
incrementally, training in order from the easiest to the hardest model.
When learning on $\Model_i$, our method exploits the knowledge acquired from its abstraction by applying of a form of Reward Shaping, constructed from the estimated solution for~$\Model_{i+1}$.
In particular, we recognize that the optimal value function $V^*_{i+1}$ of $\Model_{i+1}$ is a
helpful heuristic that can be used to evaluate how desirable a group of states is according to the abstraction.
We formalize our application of RS in the following definition.
\begin{definition}
	Let $\Model_i$ be an MDP and $\langle \Model_{i+1}, \Mapping_i \rangle$ its abstraction.
	We define the \emph{biased~MDP} of $\Model_i$ with respect to $\langle \Model_{i+1}, \Mapping_i \rangle$
	as the model $\Biased\Model_i$, resulting
	from the application of reward shaping to $\Model_i$, using the potential:
	\begin{equation}
		\Potential(s) \coloneqq V_{i+1}^*(\Mapping_i(s))
		\label{eq:shaping-potential}
	\end{equation}
	where, $V_{i+1}^*$ is the optimal value function of~$\Model_{i+1}$.
	\label{def:biased-mdp}
\end{definition}

This choice is a novel contribution of this paper and it allows to evaluate each state
according to how much desirable the corresponding abstract state is, according to the optimal policy for~$\Model_{i+1}$.
This is beneficial, as high potentials are associated to high Q-function value initializations~\cite{wiewiora2003potential}.
In fact, \cite{ng1999policy} was the first to notice that the MDP own optimal value function is a very natural potential for RS.
We extend this idea, by using the value function of the abstract MDP instead.

\section{Reward Shaping for Episodic RL}
\label{sec:shaping-episodes}

Potential-Based RS has been explicitly designed not to alter the optimal
policies. In fact, regardless of the potential, in case of an infinite horizon, or
if the episodes always terminate in a zero-potential absorbing state, this is always
guaranteed~\cite{ng1999policy}.  However, in RL, it is extremely common to diversify the agent experiences by
breaking up exploration in episodes of finite length. Thus, in the episodic
setting, these guarantees do not hold anymore, as the episodes might terminate
in states with arbitrary potential and the optimal policy can be altered~\cite{grzes2017reward}.

To see this, consider an episode $\Trace = s_0 a_1 r_1 s_1 \dots s_n$ of an
MDP~$\Model$, and the associated episode $\Trace' = s_0 a_1 r'_1 s_1 \dots s_n$, where
rewards have been modified via reward shaping. The returns of the two sequences
are related by~\cite{grzes2017reward}:
\begin{align}
	\ReturnSym(\Trace') \coloneqq \sum_{t=0}^{n-1} \gamma^t\, r'_{t+1} = \ReturnSym(\Trace) + \gamma^n\, \Potential(s_n) - \Potential(s_0)
	\label{eq:episode-return}
\end{align}
The term $\gamma^n\, \Potential(s_n)$ is the one responsible of modifying the optimal policies, as it depends on the state that is reached at the end of the episode.
So, the solution proposed by~\cite{grzes2017reward} is to assume, for every terminal state, the null potential
${\Potential(s_n) = 0}$, as this would preserve the original returns.

However, this is not always the only desirable solution.
In fact, we might be interested in relaxing the guarantee of an identical policy, in favour of a stronger impact on learning speed.
The same need has been also identified by~\cite{schubert_2021_PlanbasedRelaxed} and addressed with different solutions.
As an example, let us consider an MDP with a null reward function everywhere, except when transitioning to a distinct goal state.
As a consequence of equation~\eqref{eq:episode-return} and the choice $\Potential(s_n) = 0$, all finite trajectories which do not contain the goal state are associated to the same return.
Since the agent cannot estimate its distance to the goal through differences in return, return-invariant RS of~\cite{grzes2017reward} does not provide a \emph{persistent} exploration bias to the agent.
The form of RS adopted in this paper, which is formulated in Definition~\ref{def:biased-mdp}, does not assign null potentials to terminal states.
Therefore, we say that it is \emph{not} return-invariant.
This explains why the MDP of Definition~\ref{def:biased-mdp} has been called \mbox{``biased''}:
optimal policies of $\Biased{\Model}_i$ and $\Model_i$ do not necessarily correspond.
This is addressed in the next section, where we show that the complete procedure recovers optimal convergence.

\subsection{The Algorithm}
\label{sec:algorithm}

Since we deliberately adopt a form of RS which is not return invariant in the episodic setting, we devised a technique to recover optimality.
The present section illustrates the proposed method and proves that it converges to the optimal policy of the original MDP, when coupled with any off-policy algorithm.

\begin{algorithm}[tb]
	\caption{Main algorithm}
	\label{alg:main}
	\KwIn{Off-policy learning algorithm $\RlAlgo$}
	\KwIn{$\Model_0, \dots, \Model_{n},\,\Mapping_0, \dots, \Mapping_{n-1}$}
	\KwOut{$\Est{\Policy}^*_0$, ground MDP estimated optimum}
	\BlankLine
	$\Est{V}^*_{n+1}: s \mapsto 0$\;
	$\Mapping_n: s \mapsto s$\;
	\BlankLine
	\ForEach{$i \in \{n, \dots, 0\}$}{
	    $F_i \gets \FMakeShaping(\gamma_i, \Mapping_i, \Est{V}^*_{i+1})$\nllabel{alg:main:shaping}\;
	    $\DLearner_i \gets \RlAlgo(\Model_i)$\;
	    $\Biased{\DLearner}_i \gets \RlAlgo(\Model_i)$\;
	    \While{not $\DLearner_i.\FStopCondition()$}{
	        $s \gets \Model_i.\FState()$\;
	        $a \gets \Biased{\DLearner}_i.\FAction(s)$\nllabel{alg:main:action}\;
	        $r, s' \gets \Model_i.\FAct(a)$\;
	        $r^b \gets r + F_i(s, a, s')$\;
	        $\Biased{\DLearner}_i.\FUpdate(s, a, r^b, s')$\;
	        $\DLearner_i.\FUpdate(s, a, r, s')$\;
	    }
	    $\Est{\Policy}^*_i \gets \DLearner_i.\FOutput()$\;
	    $\Est{V}^*_i \gets \FComputeValue(\Est{\Policy}^*_i, \Model_i)$\;
	}
\end{algorithm}

The procedure is presented in detail in Algorithm~\ref{alg:main}.
Learning proceeds sequentially, training from the most abstract model to the ground domain.
When learning on the $i$-th MDP, the estimated optimal value function $\Est{V}^*_{i+1}$ of the previous model is used to obtain a reward shaping function (line~\ref{alg:main:shaping}).
This implicitly defines the biased MDP $\Biased{\Model}_i$ of Definition~\ref{def:biased-mdp}.
Experience is collected by sampling actions according to a stochastic exploration policy, as determined by the specific learning algorithm~$\RlAlgo$.
Such policy may be derived from the current optimal policy estimate for $\Biased{\Model}_i$, such as an $\epsilon$-greedy exploration policy in $\BiasedStar{\Est{\Policy}}_i$ (line~\ref{alg:main:action}).
Finally, the output of each learning phase are the estimates $\Est{\Policy}^*_i, \Est{V}^*_i$ for the original MDP~$\Model_i$.
This allows to iterate the process with an unbiased value estimate, or conclude the procedure with the final learning objective~$\Est{\Policy}^*_0$.

\begin{proposition}
    Let us consider MDPs $\Model_0, \dots, \Model_{n}$ and their associated mapping functions $\Mapping_0, \dots, \Mapping_{n-1}$.
    If $\RlAlgo$ is an off-policy learning algorithm, then,
    in every $i$-th iteration of Algorithm~\ref{alg:main}, $\Est{\Policy}^*_i$ converges to $\Policy^*_i$,
    as the number of environment interactions increase.
		\label{prop:optimal-convergence}
\end{proposition}
\begin{proof}All proofs are in the appendix.\end{proof}

\section{Abstraction Quality}
\label{sec:quality}

Our approach gives great flexibility in selecting an abstraction. 
Still, given some ground MDP, not all models are equally helpful as abstractions.
This section serves to define what properties should good abstractions possess.
As we can see from Algorithm~\ref{alg:main}, they are used to construct effective exploration policies (row~\ref{alg:main:action}).
Therefore, they should induce a biased MDP that assigns higher rewards to regions of the state space from which the optimal policy of the original problem can be easily estimated.
The exploration loss of an abstraction, introduced in Definition~\ref{def:exploration-loss}, captures this idea.

Although we may apply the proposed method in generic MDPs, our analysis focuses on a wide class of tasks that can be captured with goal MDPs.
\begin{definition}
	We say that an MDP $\Model = \langle \States, \Actions, T, R, \gamma \rangle$ is a \emph{goal MDP} iff there exists a set of goal states $\Goals \subseteq \States$ such that:
	\begin{align}
		R(s, a, s') &= 1 \quad \text{if $s \not\in \Goals$ and $s' \in \Goals$, \quad 0 otherwise}\\
		V^*(s) &= 0 \qquad \forall s \in \Goals
		\label{eq:goal-potential}
	\end{align}
	\label{def:goal}
\end{definition}
Equation~\eqref{eq:goal-potential} simply requires that from any goal state, it is not possible to re-enter any other goal and collect an additional reward.
Goal MDPs are very straightforward definitions of tasks, which are also sufficiently general, as we see in the experimental section.
\begin{assumption}
	The ground MDP $\Model_0$ is a goal MDP.
	\label{as:goal}
\end{assumption}
\begin{assumption}
	Given a goal MDP $\Model_i$, with goal states~$\Goals_i$, and its abstraction $\langle \Model_{i+1}, \Mapping_i \rangle$,
	we assume $\Model_{i+1}$ is a goal MDP with $\Goals_{i+1}$ satisfying:
	\begin{equation}
		\Goals_i = \cup_{s \in \Goals_{i+1}}  \; \Mapping^{-1}(s) 
	\end{equation}
	\label{as:abstract-goal}
\end{assumption}
In other words, abstract goals should correspond through the mapping to all and only the goal states in the ground domain.
In the example of Figure~\ref{fig:rooms8}, the gray cells in the ground MDP are mapped to the abstract goal labelled as G.

We start our analysis with two observations.
First, due to how the framework is designed, convergence on any model $\Model_i$, does depend on its abstraction, $\Model_{i+1}$, but not on any other model in the hierarchy.
Therefore, when discussing convergence properties, it suffices to talk about a generic MDP $\Model = \langle \States, \Actions, T, R, \gamma \rangle$ and its direct abstraction, $\Abst{\Model} = \langle \Abst{\States}, \Abst{\Actions}, \Abst{T}, \Abst{R}, \Abst{\gamma} \rangle$.
Let $\Mapping: \States \to \Abst{\States}$ denote the relevant mapping.
Second, while a goal MDP $\Model$ has sparse rewards, the reward function of the biased MDP $\Biased\Model$ is no longer sparse.
Depending on the abstraction $\langle \Abst\Model, \Mapping \rangle$, from which it is defined, the rewards of $\Biased\Model$ can be as dense as needed.
As confirmed empirically, this allows to achieve a faster convergence on the biased MDP.
However, its optimum $\BiasedStar\Policy$ should also be a good exploration policy for the original domain.
Therefore, we measure how similar $\BiasedStar{\Policy}$, the optimal policy for the biased MDP~$\Biased\Model$, is with respect to some optimal policy $\Policy^*$ of~$\Model$.

\begin{definition}
	Given an MDP $\Model$, the \emph{exploration loss} of an abstraction $\langle \Abst\Model, \Mapping \rangle$ is the expected value loss of executing in~$\Model$ the optimal policy of $\Biased\Model$, $\BiasedStar\Policy$, instead of its optimal policy:
	\begin{equation}
		L(\Model, \langle \Abst\Model, \Mapping \rangle) \coloneqq \max_{s \in \States}\,
		\abs{V^*(s) - Q(s, \BiasedStar\Policy)}
	\end{equation}
	\label{def:exploration-loss}
\end{definition}

In order to limit this quantity, we relate abstract states $\Abst{s} \in \Abst{\States}$ to sets of states $\Mapping^{-1}(\Abst{s}) \subseteq \States$ in the ground MDP.
Similarly, actions $\Abst{a} \in \Abst\Actions$ correspond to non-interruptible policies that only terminate when leaving the current block.
So, a more appropriate correspondence can be identified between abstract actions and $\Mapping$-relative {options} in~$\Model$.
We start by deriving, in equation~\eqref{eq:optq-multistep}, the multi-step value of a $\Mapping$-relative option in goal MDPs.

\paragraph{Multi-step value of options}

By combining the classic multi-step return of options~\cite{sutton1999between},
$\Mapping$-relative options from~\cite{abel_2020_ValuePreserving} and goal
MDPs of Definition~\ref{def:goal}, we obtain:
\begin{lemma}
    Given a goal MDP $\Model$ and $\Mapping: \States \to \Abst\States$,
    for any $s \in \States$ and $\Mapping$-relative option~$o$, the optimal value of~$o$ is:
    \begin{equation}
        \begin{aligned}
    	Q^*(s, o) &=
        \sum_{k=0}^{\infty} \gamma^k\, \sum_{s_{1:k} \in \Mapping(s)^k} \sum_{s' \not\in \Mapping(s)} \bigl(\\
    	&p(s_{1:k} s' \Given s, \Policy_o)\, \bigl( \Indicator(s' \in G) + \gamma\, V^*(s') \bigr)\bigr)
    	\end{aligned}
    	\label{eq:optq-multistep}
    \end{equation}
		\label{lem:multi-step}
\end{lemma}
This expression sums over any sequence of states ${s_1\dots s_k}$ remaining within $\Mapping(s)$ and leaving the block after $k$~steps at~$s'$.
A similar result was derived by \cite{abel_2020_ValuePreserving} for a slightly different definition for goal MDPs.
However, equation~\eqref{eq:optq-multistep} is not an expression about abstract states, yet, because it depends on the specific ground state~$s'$ that is reached at the end of the option.
Therefore, in the following definition, we introduce a parameter~$\PHomogeneity$ that quantifies how much the reachable states~$s'$ in each block are dissimilar in value.
This allows to jointly talk about the value of each group of states as a whole.
We define a function $W_\PHomogeneity: \Abst\States \times \Abst\States \to \Reals$, that, given a pair of abstract states $\Abst{s}, \Abst{s}'$, predicts, with $\PHomogeneity$-approximation error, the value of any successor ground state $s' \in \Abst{s}'$ that can be reached from some~$s \in \Abst{s}$.

\begin{definition}
	Consider an MDP~$\Model$ and an abstraction $\langle \Abst\Model, \Mapping \rangle$.
	We define the \emph{abstract value approximation} as the smallest $\PHomogeneity \ge 0$ such that
	there exists a function $W_{\PHomogeneity}: {\Abst\States \times \Abst\States \to \Reals}$
	which, for all $\Abst{s}, \Abst{s}' \in \Abst\States$, satisfies:
	\begin{multline}
		\forall s \in \Abst{s},\, \forall s' \in \Abst{s}',\, \forall a \in \Actions\\
		T(s, a, s') > 0 \implies  \abs{W_\PHomogeneity(\Abst{s}, \Abst{s}') - V^*(s')} \le \PHomogeneity
		\label{eq:equipotentials}
	\end{multline}
	\label{def:equipotentials}
\end{definition}
According to this definition, the frontier separating any two sets of states in the partition induced by $\Mapping$
must lie in ground states that can be approximated with the same optimal value, with a maximum error~$\PHomogeneity$.
Thus, any small $\PHomogeneity$ puts a constraint on the mapping function~$\Mapping$.
In the example of Figure~\ref{fig:rooms8}, each room is connected to each other though a single location, so this condition is simply satisfied for $\PHomogeneity = 0$.
However, this definition allows to apply our results in the general case,
provided that the ground states between two neighboring regions can be approximated to be equally close to the goal, with a maximum error of~$\PHomogeneity$.

Thanks to Definition~\ref{def:equipotentials}, it is possible to bound the value of options, only taking future abstract states into consideration (Lemma~\ref{lem:marginalized-optq}).
For this purpose, when starting from some \mbox{$s \in \States$},
we use $p(\FromBlockTo{s}{k}{\Abst{s}'} \Given s)$ to denote the probability of the event of remaining for~$k$ steps in the same block as~$s$,
then reaching any \mbox{$s' \in \Abst{s}'$} at the next transition.
\begin{lemma}
	Let $\Model$ be an MDP and $\langle \Abst{\Model}, \Mapping \rangle$ its abstraction, satisfying assumptions~\ref{as:goal} and \ref{as:abstract-goal}.
	The value of any $\Mapping$-relative option~$o$ in~$\Model$ admits the following lower bound:
	\begin{multline}
		Q^*(s, o) \ge
			\sum_{\Abst{s}' \in \Abst\States \setminus \{\Mapping(s)\}} \sum_{k=0}^{\infty}
					\gamma^k\, p(\FromBlockTo{s}{k}{\Abst{s}'} \Given s, \Policy_{o})\,
					\bigl(\\
			\Indicator(\Abst{s}' \in \Abst{\Goals}) + \gamma\, (W_\PHomogeneity(\Mapping(s), \Abst{s}') - \PHomogeneity) \bigr)
	\end{multline}
	at any $s \in \States$, where,~$\PHomogeneity$ and~$W_\PHomogeneity$\, follow Definition~\ref{def:equipotentials}.
	\label{lem:marginalized-optq}
\end{lemma}
This lemma provides a different characterization of options, in terms of abstract states, so that it can be exploited to obtain Theorem~\ref{th:exploration-loss-bound}.

\paragraph{Exploration loss of abstractions}

Thanks to the results of the previous section, we can now provide a bound for the exploration loss of Definition~\ref{def:exploration-loss}, for any abstraction.
We expand the results from~\cite{abel_2020_ValuePreserving} to limit this quantity.

First, we observe that any policy can be regarded as a finite set of $\Mapping$-relative options whose initiation sets are partitioned according to~$\Mapping$.
Then, from Lemma~\ref{lem:marginalized-optq}, we know that an approximation for the value of options only depends on the $k$-step transition probability to each abstract state.
So, we assume this quantity is bounded by some~$\POptAccurate$:
\begin{definition}
    Given an MDP $\Model$, a function $\Mapping: \States \to \Abst\States$ and two policies $\Policy_1$, $\Policy_2$,
    we say that $\Policy_1$ and $\Policy_2$ have \emph{abstract similarity} $\POptAccurate$ if
	\begin{multline}
		\forall s \in \States,\quad \forall \Abst{s}' \in \Abst{\States} \setminus \{\Mapping(s)\}, \quad
		\forall k \in \Naturals \\
		\abs{\,
			p(\FromBlockTo{s}{k}{\Abst{s}'} \Given s, \Policy_1) -
			p(\FromBlockTo{s}{k}{\Abst{s}'} \Given s, \Policy_2)\,
		} \le \POptAccurate
	\end{multline}
	\label{def:similar-policies}
\end{definition}
Intuitively, abstract similarity measures the difference between the two abstract actions described by each policy,
as it only depends on the probability of the next abstract state that is reached, regardless of the single trajectories and the specific final ground state.
In the running example of Figure~\ref{fig:rooms8},
two policies with low~$\POptAccurate$, after the same number of steps, would reach the same adjacent room with similar probability.
It is now finally possible to state our result.
\begin{theorem}
	Let $\Model$ and $\langle \Abst\Model, \Mapping \rangle$ be and MDP and an abstraction satisfying assumptions~\ref{as:goal} and \ref{as:abstract-goal},
	and let $\Biased\Model$ be the biased MDP.
	If $\epsilon$ is the abstract similarity of $\Policy^*$ and $\BiasedStar{\Policy}$, and the abstract value approximation is~$\PHomogeneity$, then, the exploration loss of $\langle \Abst\Model, \Mapping \rangle$ satisfies:
	\begin{equation}
		L(\Model, \langle \Abst\Model, \Mapping \rangle) \le
			\frac{2 \abs{\Abst\States} (\POptAccurate + \gamma\, \PHomogeneity)}{(1-\gamma)^2}
		\label{eq:exploration-loss-bound}
	\end{equation}
	\label{th:exploration-loss-bound}
\end{theorem}

This theorem shows under which conditions an abstraction induces an exploration policy that is similar to some optimal policy of the original domain.
However, we recall that optimal convergence is guaranteed regardless of the abstraction quality, because the stochastic exploration policy satisfies the mild conditions posed by the off-policy learning algorithm adopted.

Apart from our application to the biased MDP policy, the theorem has also a more general impact.
The result shows that, provided that the abstraction induces a partition of states whose frontiers have some homogeneity in value (Definition~\ref{def:equipotentials}), it is possible to reason in terms of abstract transitions.
Only for a $\PHomogeneity=0$, this bound has similarities with the inequality n.~5\ in
\cite{abel_2020_ValuePreserving}. Notice however, that the one stated here is expressed
in terms of the size of the abstract state space, which usually can
be assumed to be $\abs{\Abst\States} \ll \abs{\States}$.

\section{Validation}
\label{sec:validation}

We initially consider a navigation scenario, where some locations in a map are selected as goal states.

\paragraph{Environments}
We start with two levels of abstractions, $\Model_1$ and~$\Model_2$.
The ground MDP $\Model_1$ consists of a finite state space $\States_1$,
containing a set of locations, and a finite set of actions $\Actions_1$ that
allows to move between neighboring states, with some small failure probability.
Following the idea of Figure~\ref{fig:rooms8}, we also define an abstract MDP~$\Model_2$,
whose states correspond to contiguous regions of~$\States_1$.
Actions~$\Actions_2$ allow to move, with high probability, from any region to any other, only if there is a direct connection in~$\Model_1$.
We instantiate this idea in two domains. In the first, we consider a map as the one in the classic 4-rooms environment from \cite{sutton1999between}.
The second, is the ``8-rooms'' environment shown in Figure~\ref{fig:rooms8}.%
\footnote{Code available at {https://github.com/cipollone/multinav2}}

\paragraph{Training results}

In the plots of Figures~\ref{fig:4rooms-steps} and \ref{fig:8rooms-steps}, for each of the two ground MDPs, we compare the performance of the following algorithms:
\begin{description}
	\item[{\tikz [baseline=-0.5ex] \draw [no-rs, thick] (0,0) -- +(1.5em,0);}]
		Q-learning~\cite{watkins1992q};
	\item[{\tikz [baseline=-0.5ex] \draw [delayedq-rs, thick] (0,0) -- +(1.5em,0);}]
		Delayed Q-learning~\cite{strehl2006pac};
	\item[{\tikz [baseline=-0.5ex] \draw [our-rs, thick] (0,0) -- +(1.5em,0);}]
		Algorithm~\ref{alg:main} (our approach) with Q-learning.
\end{description}

\begin{figure}
	\centering
	\subfloat[][Navigation task in the 4-rooms domain.\label{fig:4rooms-steps}]%
	{\includegraphics{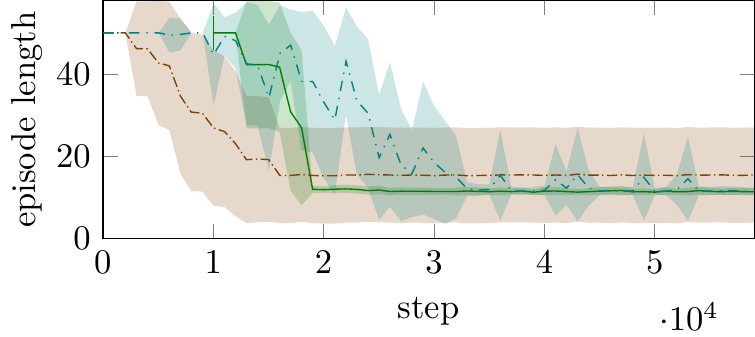}}\\
	\subfloat[][Navigation task in the 8-rooms domain.\label{fig:8rooms-steps}]%
	{\includegraphics{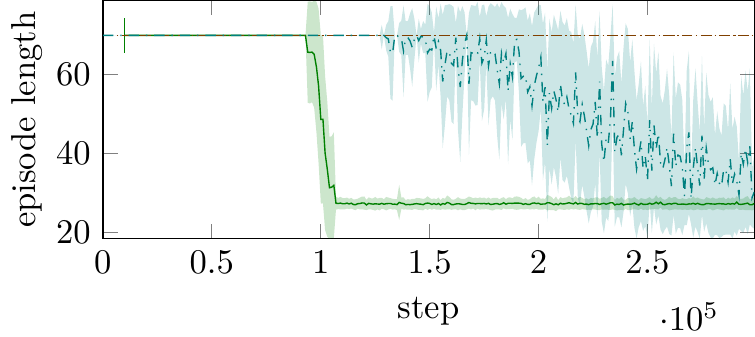}}
	\caption{Results on the navigation tasks.}
\end{figure}

Each episode is terminated after a fixed timeout or when the agent reaches a goal state.
Therefore, lower episode lengths are associated to higher cumulative returns.
Horizontal axis spans the number of sampled transitions.
Each point in these plots shows the average and standard deviation of the evaluations of 10 different runs.
The solid green line of our approach is shifted to the right, so to account for the number of time steps that were spent in training the abstraction.
Further training details can be found in the appendix.
As we can see from Figure~\ref{fig:4rooms-steps}, all algorithms converge relatively easily in the smaller 4-rooms domain.
In Figure~\ref{fig:8rooms-steps}, as the state space increases and it becomes harder to explore, a naive exploration policy does not allow Q-learning to converge in reasonable time.
Our agent, on the other hand, steadily converge to optimum, even faster than Delayed Q-learning which has polynomial time guarantees.

\subsection{Return-Invariant Shaping}

As discussed in Section~\ref{sec:shaping-episodes}, when applying RS in the episodic setting,
there is a technical but delicate distinction to make between:
\begin{description}
	\item[{\tikz [baseline=-0.5ex] \draw [inv-rs, thick] (0,0) -- +(1.5em,0);}]
		Return-invariant RS (null potentials at terminal states);
	\item[{\tikz [baseline=-0.5ex] \draw [our-rs, thick] (0,0) -- +(1.5em,0);}]
		Non return-invariant RS (our approach).
\end{description}
In Figure~\ref{fig:8rooms-inv} (top), we compare the two variants on the 8-rooms domain.
\begin{figure}
	\centering
	{\includegraphics{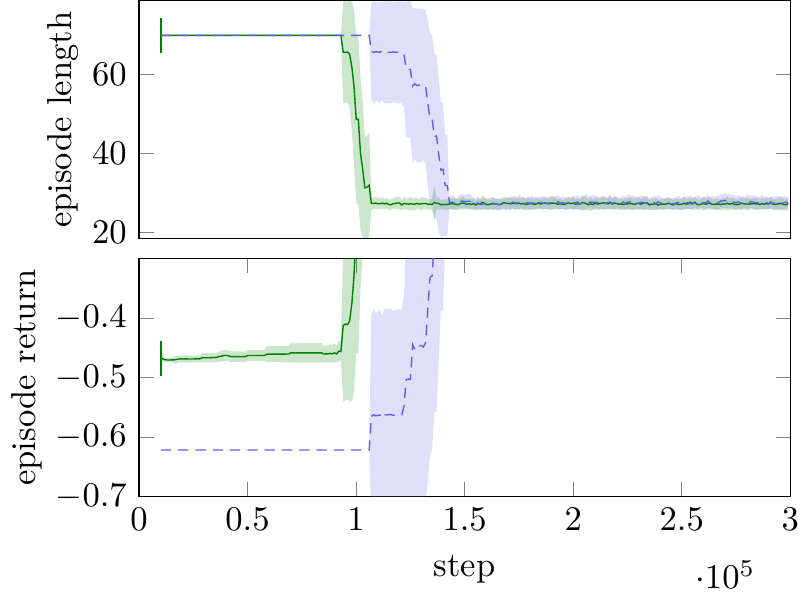}}
	\caption{Return-invariant RS and our approach.}
	\label{fig:8rooms-inv}
\end{figure}
Although both agents receive RS from the same potential,
this minor modification suffices to produce this noticeable difference.
The reason lies in the returns the two agents observe (bottom).
Although they are incomparable in magnitude, in the early learning phase,
we see that only our reward shaping is able to reward each episode differently,
depending on their estimated distance to the goal.

\subsection{Robustness to Modelling Errors}

We also considered the effect of significant modelling errors in the abstraction.
In Figure~\ref{fig:reach-abs-errors}, we report the performance of our agent on the 8-rooms domain, when driven by three different abstractions:
\begin{figure}
\centering
	\includegraphics{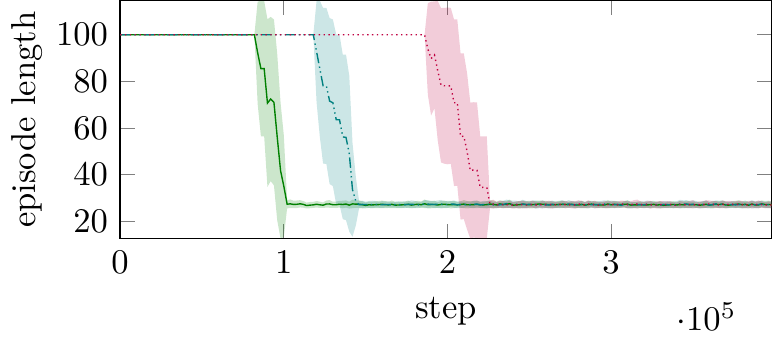}
	\caption{Training in presence of errors.}
    \label{fig:reach-abs-errors}
\end{figure}
\begin{description}
\item[{\tikz [baseline=-0.5ex] \draw [our-rs, thick] (0,0) -- +(1.5em,0);}]
	$\Model_2$: is the same abstraction used in Figure~\ref{fig:8rooms-steps};
\item[{\tikz [baseline=-0.5ex] \draw [bad1-rs, thick] (0,0) -- +(1.5em,0);}]
	$\Model_2^{(b)}$: is $\Model_2$ with an additional transition from the pink states (p) to the goal (G),
	not achievable in $\Model_1$.
\item[{\tikz [baseline=-0.5ex] \draw [bad2-rs, thick] (0,0) -- +(1.5em,0);}]
	$\Model_2^{(c)}$: is $\Model_2^{(b)}$ with an additional
	transition from the blue (b) to the pink region (p),
	not achievable in $\Model_1$.
\end{description}
Clearly, abstractions with bigger differences with respect to the underlying domain cause the learning process to slow down.
However, with any of these, Q-learning converges to the desired policy and the performance degrades gracefully.
Interestingly, even in presence of severe modelling errors, the abstraction still provides useful information with respect to uninformed exploration.

\subsection{Interaction Task}

In this section, we demonstrate that the proposed method applies to a wide range of algorithms, dynamics and tasks.
With respect to variability in tasks, we emphasize that goal MDPs can capture many interesting problems.
For this purpose, instead of reaching a location, we consider a complex temporally-extended behavior such as:
``reach the entrance of the two rooms in sequence and, if each door is open, enter and interact with the person inside, if present''.
This task is summarized by the deterministic automaton~$\Automa$ of Figure~\ref{fig:office-automa}.
Note that there is a single accepting state, and arcs are associated to environment events.
\begin{figure}
    \centering
		\begin{tikzpicture}[
				state/.style={state without output, inner sep=0pt, minimum size=10pt},
				label/.style={font=\scriptsize},
				>=stealth'
			]
			\matrix [row sep=3mm, column sep=5em] {
				\&\& \node (q2) [state] {}; \&\\
				\node (q0) [state] {}; \&
				\node (q1) [state] {}; \&\&
				\node (q4) [state] {}; \\
				\&\& \node (q3) [state] {}; \&\\
			};
			\draw (q0) edge [loop above] node [label, above] {$\neg \textit{Out}_i$} (q0);
			\draw (q0) edge [->] node [label, below] {$\textit{Out}_i \land \neg \textit{Closed}$} (q1);
			\draw (q1) edge [loop above] node [label, above] {$\neg \textit{In}_i$} (q1);
			\draw (q1) edge [->] node [label, below, xshift=4mm] {$\textit{In}_i \land \textit{Person}$} (q2);
			\draw (q1) edge [->] node [label, below, xshift=-2mm, yshift=-1mm] {$\textit{In}_i \land \neg \textit{Person}$} (q3);
			\draw (q3) edge [->] node [label, below] {$\neg \textit{Talking}$} (q4);
			\draw (q2) edge [->] node [label, below] {$\textit{Talking}$} (q4);
			\draw (q0) edge [bend left=40, ->] node [label, above] {$\textit{Closed}$} (q4);
			\draw [<-, dashed] (q0.west) -- +(-5mm, 0);
			\node (qf) [state, accepting, right=5mm of q4] {};
			\draw [->, dashed] (q4) -- (qf);
    \end{tikzpicture}
		\caption{A temporally-extended task, repeated for $i = 1, 2$.
		The missing transitions go to a failure sink state.}
    \label{fig:office-automa}
\end{figure}
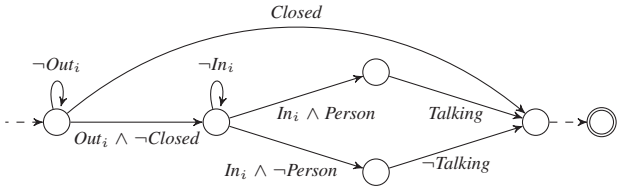

Regarding the environment dynamics, instead, we define $\Model_{2,d}$ and $\Model_{1,d}$, respectively the abstract and grid transition dynamics seen so far.
In addition, we consider a ground MDP~$\Model_{0,d}$ at which the robot movements are modelled using continuous features.
The state space $\States_0$ now contains continuous vectors $(x, y, \theta, v)$,
representing pose and velocity of agent's mobile base on the plane.
The discrete set of actions $\Actions_0$ allows to accelerate, decelerate, rotate, and a special action denotes the initiation of an interaction.

There exists goal MDPs, $\Model_2, \Model_1, \Model_0$ that capture both the dynamics and the task
defined above, which can be obtained through a suitable composition of each $\Model_{i,d}$ and~$\Automa$ \cite{brafman_2018_LTLfLDLf,icarte2018using}.
Therefore, we can still apply our tecnique to the composed goal MDP.
Since $\Model_0$ now includes continuous features we adopt Dueling DQN~\cite{wang_dueling_2016}, a Deep RL algorithm.
The plot in Figure~\ref{fig:task-cont-dqn} shows a training comparison between
the Dueling DQN agent alone (dot-dashed brown), and Dueling DQN receiving rewards from the grid abstraction (green).
\begin{figure}
	\centering
	\includegraphics{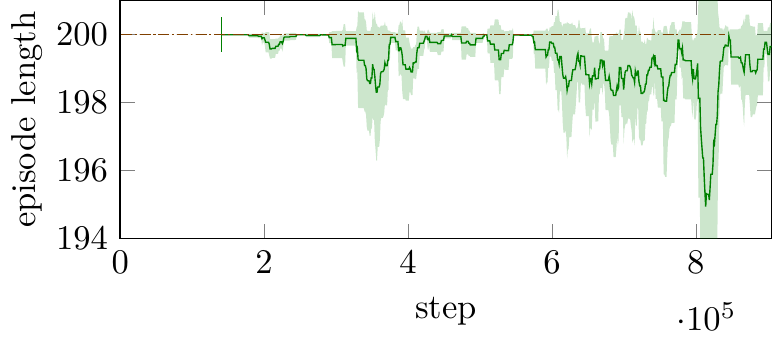}
	\caption{Dueling DQN algorithm with and without our RS. Training episode lengths, averaged over 5 runs.}
	\label{fig:task-cont-dqn}
\end{figure}
As we can see, our method allows to provide useful exploration bias even in case of extremely sparse goal states, as in this case.

\section{Related Work}

Hierarchical RL specifically studies efficiency in presence of abstractions.
Some classic approaches are MAXQ~\cite{dietterich2000hierarchical},
HAM~\cite{parr1998reinforcement} and options~\cite{sutton1999between}.
Instead of augmenting the ground MDP with options, which would result in a semi-MDP,
we use them as formalizations of partial policies.
%
In order to describe which relation should the ground MDP and its abstraction satisfy,
\cite{ravindran_model_2002, li06towards} develop MDP Homomorphisms and
approximated extensions.
Differently to these works, our method does not try to capture
spatial regularities in the domain, rather, we are interested in coarse
partitioning of neighboring states, which can hardly be approximated with
a single value.
This issue also appeared in~\cite{jothimurugan_2021_AbstractValue} in the form of non-Markovianity.
 
Our abstractions are closely related to those described
in~\cite{abel_2016_OptimalBehavior, abel_2020_ValuePreserving},
in which both the state and the action spaces differ.
Still, they do not exploit, as in our work, explicit abstract MDPs,
since they only learn in one ground model.
%
Regarding the use of Reward Shaping, \cite{gao2015potential} presented the idea of
applying RS in context of HRL, and applying it specifically to the MAXQ
algorithm.
Recently, \cite{schubert_2021_PlanbasedRelaxed} proposed a new form of biased RS for goal MDPs, with looser convergence guarantees.
With a different objective with respect to this paper,
various works consider how abstractions may be learnt instead of being pre-defined, including
\cite{Marthi07automatic-shaping, grzes2008multigrid, steccanella_hierarchical_2021}.
This is an interesting direction to follow and the models that are obtained with such techniques may be still exploited with our method.

\section{Conclusion}

In this paper, we have presented an approach to increase the sample efficiency of RL algorithms,
based on a linear hierarchy of abstract simulators and a new form of reward shaping.
While the ground MDP accurately captures the environment dynamics,
higher-level models represent increasingly coarser abstractions of it. 
We have described the properties of our RS method under different abstractions and we have shown its effectiveness in practice.
Importantly, our approach is very general, as it makes no assumptions on
the off-policy algorithm that is used, and it has minimal requirements in terms of mapping
between the abstraction layers.
As future work, we plan to compare our technique with other methods from Hierarchical RL,
especially those with similar prior assumptions and to evaluate them in a robotic application with low-level perception and control.

\section*{Acknowledgements}
This work has been supported by the ERC Advanced Grant WhiteMech (No. 834228),
by the EU ICT-48 2020 project TAILOR (No. 952215), and by the PRIN project
RIPER (No. 20203FFYLK).

\bibliography{bibliography}

\appendix
\onecolumn
\loadgeometry{articlegeometry}
\input{appendix/appendix-text.tex}

\end{document}

%% file: appendix/appendix-text.tex
\section{Proofs of formal statements}

\subsection{Proposition~\refmain{prop:optimal-convergence} -- optimal convergence}

\begin{propositionmain}{\refmain{prop:optimal-convergence}}
	Let us consider MDPs $\Model_0, \dots, \Model_{n}$ and their associated mapping functions $\Mapping_0, \dots, \Mapping_{n-1}$.
	If $\RlAlgo$ is an off-policy learning algorithm, then,
	in every $i$-th iteration of Algorithm~\refmain{alg:main}, $\Est{\Policy}^*_i$ converges to $\Policy^*_i$,
	as the number of environment interactions increase.
\end{propositionmain}
\begin{proof}
	At any iteration $i \in \{n, \dots, 0\}$ of Algorithm~\refmain{alg:main},
	we are given $\Model_i$, $\Mapping_i$ and~$\Est{V}_{i+1}^*$.
	By construction, the two instantiations of $\RlAlgo$, $\DLearner_i$ and $\Biased{\DLearner}_i$,
	perform updates from transitions generated from $\Model_i$ and $\Biased{\Model}_i$, respectively.
	Actions are selected according to $\Biased{\DLearner}_i$ which follows some exploration policy $\Biased\ExplorationPolicy$.
	Since $\Model_i$ and $\Biased{\Model}_i$ share the same state and action spaces,
	$\Biased\ExplorationPolicy$ is also an exploration policy for~$\Model_i$, in the sense of Definition~\refmain{def:off-policy-learning}.
	Therefore, ${\DLearner}_i$ also converges to $\Policy^*_i$, as the number of environment interactions $t \to \infty$.
\end{proof}

\subsection{Lemma~\refmain{lem:multi-step} -- multi-step value}

\begin{lemmamain}{\refmain{lem:multi-step}}
	Given a goal MDP $\Model$ and $\Mapping: \States \to \Abst\States$,
	for any $s \in \States$ and $\Mapping$-relative option~$o$, the optimal value of~$o$ is:
	\begin{equation}
		Q^*(s, o) =
			\sum_{k=0}^{\infty} \gamma^k \sum_{s_{1:k} \in \Mapping(s)^k} \sum_{s' \not\in \Mapping(s)}
			p(s_{1:k} s' \Given s, \Policy_o)\, \bigl( \Indicator(s' \in \Goals) + \gamma\, V^*(s') \bigr)
	\end{equation}
\end{lemmamain}
\begin{proof}
	By assumption $\Model$ is a goal MDP over some $\Goals \subseteq \States$.
	For $s \in \Goals$, we know $Q^*(s, o) = 0$. We consider $s \not\in \Goals$.
	Since the MDP $\Model$ is clear from the context, to uniform the notation, we use $p(s' \Given s, a)$ instead of $T(s, a, s')$.
	Following a similar procedure as~\cite{abel_2020_ValuePreserving}, for our definition of goal MDPs:
	\begin{align}
		{Q}^*(s, o) &\!\coloneqq 
			\Expected_{s' \Given s, o}[\, R(s, \Policy_o(s), s')\, +
				\gamma (\, \Indicator(s' \in \Mapping(s))\, {Q}^*(s', o)
				+ \Indicator(s' \not\in \Mapping(s))\, {V}^*(s'))\,] \\
			&= \sum_{s' \in \Mapping(s)} p(s' \Given s, \Policy_o(s))[\,\cdot\,] +
				\sum_{s' \not\in \Mapping(s)} p(s' \Given s, \Policy_o(s))[\,\cdot\,] \\
			&= \sum_{s' \in \Mapping(s)} p(s' \Given s, \Policy_o(s)) \,
					\gamma \, {Q}^*(s', o)\, +
				\sum_{s' \not\in \Mapping(s)} p(s' \Given s, \Policy_o(s)) \bigl(
					\Indicator(s' \in \Goals) + \gamma\, {V}^*(s')
				\bigr) 
				\label{eq:option-value-to-expand}
	\end{align}
	We abbreviate the second term of~\eqrefthis{eq:option-value-to-expand} with $\Psi$ and let $s_0 = s$.
	Then, similarly to the classic multi-step value of options~\cite{sutton1999between}, we can expand over time.
	\begin{align}
		&Q^*(s, o) = \sum_{s' \in \Mapping(s)} p(s' \Given s, \Policy_o(s)) \,
					\gamma \, {Q}^*(s', o)\, + \Psi(s, o)\\
			&= \Psi(s, o) + \gamma\, \sum_{s' \in \Mapping(s)} p(s' \Given s,
					\Policy_o(s))\,\Psi(s', o) \,+
				\gamma^2 \sum_{s', s'' \in \Mapping(s)^2} 
				p(s'\,s'' \Given s, \Policy_o(s)\,\Policy_o(s'))\, Q^*(s'', o)\\
			&= \sum_{k=0}^{\infty} \gamma^k\, \sum_{s_{1:k} \in \Mapping(s)^{k}}
				p(s_{1:k} \Given s, \Policy_o)\,\Psi(s_{k}, o)\\
			&= \sum_{k=0}^{\infty} \gamma^k\,
				\sum_{s_{1:k} \in \Mapping(s)^k} \sum_{s'\not\in \Mapping(s)}
				p(s_{1:k} s' \Given s, \Policy_o)\, \bigl(
				\Indicator(s' \in \Goals) + \gamma\, V^*(s') \bigr)
				\label{eq:options-multistep-proof-end}
	\end{align}
	which is the expression in Lemma~\refmain{lem:multi-step}.
\end{proof}

With a similar procedure, we can also show that the multi-step value of a $\Mapping$-relative
option~$o$ in a biased MDP $\Biased\Model$ with respect to $\langle \Abst\Model, \Mapping \rangle$ is:
\begin{align}
	\BiasedStar{Q}(s, o)
		&= \sum_{k=0}^{\infty} \gamma^k\,
			\sum_{s_{1:k} \in \Mapping(s)^k} \sum_{s' \not\in \Mapping(s)}
			p(s_{1:k} s' \Given s, \Policy_o) \, \bigl(
		\Indicator(s' \in \Goals) + \gamma\, \Abst{V}^*(\Mapping(s'))
			- \Abst{V}^*(\Mapping(s)) + \gamma\, V^*(s') \bigr)
\end{align}

\subsection{Lemma~\refmain{lem:marginalized-optq} -- marginalized value}
\begin{lemmamain}{\refmain{lem:marginalized-optq}}
	Let $\Model$ be an MDP and $\langle \Abst{\Model}, \Mapping \rangle$ its abstraction, satisfying assumptions~\refmain{as:goal} and \refmain{as:abstract-goal}.
	The value of any $\Mapping$-relative option~$o$ in~$\Model$ admits the following lower bound:
	\begin{equation}
		Q^*(s, o) \ge
			\sum_{\Abst{s}' \in \Abst\States \setminus \{\Mapping(s)\}} \sum_{k=0}^{\infty}
					\gamma^k\, p(\FromBlockTo{s}{k}{\Abst{s}'} \Given s, \Policy_{o})\,
					\bigl(
			\Indicator(\Abst{s}' \in \Abst{\Goals}) + \gamma\, (W_\PHomogeneity(\Mapping(s), \Abst{s}') - \PHomogeneity) \bigr)
	\end{equation}
	at any $s \in \States$, where,~$\PHomogeneity$ and~$W_\PHomogeneity$\, follow Definition~\refmain{def:equipotentials}.
\end{lemmamain}
\begin{proof}
	We recall that $p(\FromBlockTo{s}{k}{\Abst{s}'} \Given s, \Policy)$ denotes
	the probability of the event of remaining for~$k$ steps within
	$\Mapping(s)$, then reaching~$\Abst{s}'$ at the next transition, when following policy~$\Policy$, starting from~$s$.
	Similarly, we use $p(\FromBlockTo{s}{k}{{s}'} \Given s, \Policy)$ to represent
	the probability of remaining $k$~steps within $\Mapping(s)$ then reaching a specific ground state $s' \in \States \setminus \Mapping(s)$.

	To obtain the result, we marginalize the probabilities appearing in
	Lemma~\refmain{lem:multi-step} over all possible trajectories~$s_{1:k}$:
	\begin{align}
		Q^*(s, o) &= \sum_{s' \in \States \setminus \Mapping(s)} \sum_{k=0}^{\infty}
					\gamma^k\, p(\FromBlockTo{s}{k}{s'} \Given s, \Policy_{o})\, \bigl(
			\Indicator(s' \in \Goals) + \gamma\, V^*(s') \bigr)
		\label{eq:marginalized-ground-value}
	\end{align}
	Now, for all $s', k$ such that $p(\FromBlockTo{s}{k}{s'} \Given s, \Policy_{o}) > 0$,
	there is one state $s_k \in \Mapping(s)$, reachable in $k$ steps from~$s$ under~$\Policy_o$, from which $T(s_k, \Policy_o(s_k), s') > 0$.
	From Definition~\refmain{def:equipotentials}, we know $\abs{W_\PHomogeneity(\Mapping(s_k), \Mapping(s')) - V^*(s')} \le \PHomogeneity$.
	Therefore, we can provide a lower bound for each term $V^*(s')$ in the sum above:
	\begin{equation}
		Q^*(s, o) \ge \sum_{s' \in \States \setminus \Mapping(s)} \sum_{k=0}^{\infty}
					\gamma^k\, p(\FromBlockTo{s}{k}{s'} \Given s, \Policy_{o})\, \bigl(
			\Indicator(s' \in \Goals) + \gamma\, (W_\PHomogeneity(\Mapping(s), \Mapping(s')) - \PHomogeneity) \bigr)
	\end{equation}
	because $\Mapping(s_k) = \Mapping(s)$.
	It is now possible to split the sum $\sum_{s' \in \States \setminus \Mapping(s)}$ into
	${\abs{\Abst{\States}} - 1}$ sums over future blocks and marginalize among them to obtain:
	\begin{equation}
			Q^*(s, o) \ge
				\sum_{\Abst{s}' \in \Abst\States \setminus \{\Mapping(s)\}} \sum_{k=0}^{\infty}
						\gamma^k\, p(\FromBlockTo{s}{k}{\Abst{s}'} \Given s, \Policy_{o})\,
						\bigl(
				\Indicator(\Abst{s}' \in \Abst\Goals) + \gamma\, (W_\PHomogeneity(\Mapping(s), \Abst{s}') - \PHomogeneity) \bigr)
	\end{equation}
	since $\Indicator(s \in \Goals) = \Indicator(\Mapping(s) \in \Abst\Goals)$.
	This proves the lemma. With the same procedure, we also obtain the upper bound:
	\begin{equation}
			Q^*(s, o) \le
				\sum_{\Abst{s}' \in \Abst\States \setminus \{\Mapping(s)\}} \sum_{k=0}^{\infty}
						\gamma^k\, p(\FromBlockTo{s}{k}{\Abst{s}'} \Given s, \Policy_{o})\,
						\bigl(
				\Indicator(\Abst{s}' \in \Abst\Goals) + \gamma\, (W_\PHomogeneity(\Mapping(s), \Abst{s}') + \PHomogeneity) \bigr)
	\end{equation}
\end{proof}

\subsection{Theorem~\refmain{th:exploration-loss-bound} -- exploration loss}

\begin{theoremmain}{\refmain{th:exploration-loss-bound}}
	Let $\Model$ and $\langle \Abst\Model, \Mapping \rangle$ be and MDP and an abstraction satisfying assumptions~\refmain{as:goal} and \refmain{as:abstract-goal},
	and let $\Biased\Model$ be the biased MDP.
	If $\epsilon$ is the abstract similarity of $\Policy^*$ and $\BiasedStar{\Policy}$, and the abstract value approximation is~$\PHomogeneity$, then, the exploration loss of $\langle \Abst\Model, \Mapping \rangle$ satisfies:
	\begin{equation}
		L(\Model, \langle \Abst\Model, \Mapping \rangle) \le
			\frac{2 \abs{\Abst\States} (\POptAccurate + \gamma\, \PHomogeneity)}{(1-\gamma)^2}
	\end{equation}
\end{theoremmain}
\begin{proof}
	Any policy $\Policy$ can be represented as a set $\Options$, composed of $\Mapping$-relative options
	whose initiation sets are partitioned according to $\Mapping$ and $\forall s \in \States,
	\exists o \in \Options: {\Policy_o(s) = \Policy(s)}$.
	Let $\Options^*$ and $\BiasedStar\Options$ be the $\Mapping$-relative options associated to $\Policy^*$ and $\BiasedStar\Policy$.
	We now compute what is the difference in value between executing $\Policy^*$ and $\BiasedStar\Policy$,
	for one option each, then following $\Policy^*$ afterwards.
	For any $s \in \States \setminus \Goals$, let $o^*$ and $\BiasedStar{o}$ be the relevant options in $\Options^*$ and $\BiasedStar{O}$, respectively.
	We bound the following difference in value:
	\begin{align}
			\abs{Q^*(s, o^*) - Q^*(s, \BiasedStar{o})} =
			Q^*(s, o^*) - Q^*(s, \BiasedStar{o})
	\end{align}
	From an application of the upper and lower bound of Lemma~\refmain{lem:marginalized-optq},
	\begin{align}
		&\abs{Q^*(s, o^*) - Q^*(s, \BiasedStar{o})} \\
			&\le \sum_{\Abst{s}' \in \Abst\States \setminus \{\Mapping(s)\}} \sum_{k=0}^{\infty}
						\gamma^k\, p(\FromBlockTo{s}{k}{\Abst{s}'} \Given s, \Policy_{o^*})\,
						\bigl(
				\Indicator(\Abst{s}' \in \Abst\Goals) + \gamma\, (W_\PHomogeneity(\Mapping(s), \Abst{s}') + \PHomogeneity) \bigr)\, - \notag\\
			&\qquad\sum_{\Abst{s}' \in \Abst\States \setminus \{\Mapping(s)\}} \sum_{k=0}^{\infty}
						\gamma^k\, p(\FromBlockTo{s}{k}{\Abst{s}'} \Given s, \Policy_{\BiasedStar{o}})\,
						\bigl( \Indicator(\Abst{s}' \in \Abst\Goals) + \gamma\, (W_\PHomogeneity(\Mapping(s), \Abst{s}') - \PHomogeneity) \bigr)\\
			&=\sum_{\Abst{s}' \in \Abst\States \setminus \{\Mapping(s)\}}
						\bigl( \Indicator(\Abst{s}' \in \Abst\Goals) + \gamma\, W_\PHomogeneity(\Mapping(s), \Abst{s}') \bigr)
						\sum_{k=0}^{\infty}
						\gamma^k\, \bigl(\, p(\FromBlockTo{s}{k}{\Abst{s}'} \Given s, \Policy_{o^*}) - p(\FromBlockTo{s}{k}{\Abst{s}'} \Given s, \Policy_{\BiasedStar{o}}) \bigr)\, + \notag\\
			&\qquad\sum_{\Abst{s}' \in \Abst\States \setminus \{\Mapping(s)\}} \sum_{k=0}^{\infty}
						\gamma^k\, \bigl(\, p(\FromBlockTo{s}{k}{\Abst{s}'} \Given s, \Policy_{o^*}) + p(\FromBlockTo{s}{k}{\Abst{s}'} \Given s, \Policy_{\BiasedStar{o}}) \bigr)
						\, \gamma\, \PHomogeneity
	\end{align}
	Now we apply Definition~\refmain{def:similar-policies} of abstract similarity and bound
	\begin{align}
		\abs{Q^*(s, o^*) \,&-\, Q^*(s, \BiasedStar{o})} \\
			&\le \sum_{\Abst{s}' \in \Abst\States \setminus \{\Mapping(s)\}}
					\bigl( \Indicator(\Abst{s}' \in \Abst\Goals) + \gamma\, W_\PHomogeneity(\Mapping(s), \Abst{s}') \bigr)
			\sum_{k=0}^{\infty}
						\gamma^k\, \POptAccurate\, +
			\sum_{\Abst{s}' \in \Abst\States \setminus \{\Mapping(s)\}} \sum_{k=0}^{\infty}
						\gamma^{k+1}\, 2 \, \PHomogeneity\\
			&= \sum_{\Abst{s}' \in \Abst\States \setminus \{\Mapping(s)\}}
					\Bigl(\bigl( \Indicator(\Abst{s}' \in \Abst\Goals) + \gamma\, W_\PHomogeneity(\Mapping(s), \Abst{s}') \bigr)
					\frac{\POptAccurate}{1-\gamma} +
					\frac{2\, \gamma\, \PHomogeneity}{1-\gamma} \Bigr)
	\end{align}
	Since in a goal MDP the maximum value is 1, we know 
	$W_\PHomogeneity(\Mapping(s), \Abst{s}') \le (1 + \PHomogeneity)$, for all $s \in \States, \Abst{s}' \in \Abst\States$.
	Moreover, if $W_\PHomogeneity$ satisfies condition \eqrefmain{eq:equipotentials},
	the function $W_{\PHomogeneity}^{\text{clip}}(\Abst{s}, \Abst{s}') \coloneqq \min\{1, W_\PHomogeneity(\Abst{s}, \Abst{s}')\}$
	also satisfies it. Concluding,
	\begin{align}
			\abs{Q^*(s, o^*) - Q^*(s, \BiasedStar{o})}
			&\le \sum_{\Abst{s}' \in \Abst\States \setminus \{\Mapping(s)\}}
					\Bigl(\bigl( 1 + \gamma \bigr)
					\frac{\POptAccurate}{1-\gamma} +
					\frac{2\, \gamma\, \PHomogeneity}{1-\gamma} \Bigr)\\
			&\le \abs{\Abst\States}
					\Bigl(
					\frac{2\, \POptAccurate}{1-\gamma} +
					\frac{2\, \gamma\, \PHomogeneity}{1-\gamma} \Bigr)\\
			&= \frac{2 \abs{\Abst\States} (\POptAccurate + \gamma\, \PHomogeneity)}{1-\gamma}
	\end{align}
	This is a bound on the value loss of executing a single option from the set $\BiasedStar\Options$.
	To this option set the results from \cite{abel_2020_ValuePreserving} apply and equation~(3), in particular.
	So we obtain the final result:
	\begin{equation}
		L(\Model, \langle \Abst\Model, \Mapping \rangle) \le
			\frac{2 \abs{\Abst\States} (\POptAccurate + \gamma\, \PHomogeneity)}{(1-\gamma)^2}
	\end{equation}
\end{proof}

\section{Experimental details}

\subsection{Details for each plot}

In this section, we provide details and hyper-parameters for each plot of the main paper.
An even more comprehensive list can be found in the subdirectories of the archive \verb|experiments.tar.bz2|,
that is distributed together within the software repository\footnote{\texttt{https://github.com/cipollone/multinav2}},
at release \verb|v0.2.7|.
The structure of this archive will be discussed in section~\refthis{sec:using}.

\subsubsection*{Figure~\refmain{fig:4rooms-steps}}
\paragraph{Environment}
The ground environment, $\Model_1$, is the 4-rooms map appearing in~\cite{sutton1999between,abel_2020_ValuePreserving}.
Transitions have 4\% failure probability (another random action is executed instead).
The agent starts in the lower-left room, coordinate $(1, 9)$ (vertical coordinate increases downward).
Episode maximum length 50 steps, discounting $\gamma_1 = 0.98$.
The abstraction, $\Model_2$, is a 4-states environment with 10\% failure probability, discounting $\gamma_2 = 0.9$.

\paragraph{Algorithms}
Each point in the plot shows average and standard deviation computed over evaluations of 10 different runs.
Each evaluation computes the average episode length (time to goal or to timeout) from 10 episodes.
\begin{itemize}
	\item DelayedQ algorithm,
		$\epsilon_1 = 0.01$, $\delta = 0.1$, $\text{MaxR} = 1.0$, $m = 15$,
		60000 timesteps.
	\item Q-learning,
		learning rate decay from 0.1 to 0.02, 
		$\epsilon$-greedy exploration decay from 1.0 to 0.0.
		60000 timesteps.
	\item Q-learning with our Reward Shaping,
		learning rate decay from 0.1 to 0.02,
		$\epsilon$-greedy exploration decay from 1.0 to 0.0.
		50000 timesteps.
		The plot shows the performance of~$\Est{\Policy}_1^*$.
\end{itemize}

\subsubsection*{Figure~\refmain{fig:8rooms-steps}}
\paragraph{Environment}
The ground environment, $\Model_1$, is the 8-rooms map of Figure~\refmain{fig:rooms8}.
Transitions have 4\% failure probability.
The agent starts in the upper-left room, coordinate $(2, 3)$.
Episode maximum length 70 steps, discounting $\gamma_1 = 0.98$.
The abstraction, $\Model_2$, is a 8-states environment with 10\% failure probability, discounting $\gamma_2 = 0.9$.

\paragraph{Algorithms}
Each point in the plot shows average and standard deviation computed over evaluations of 10 different runs.
Each evaluation computes the average episode length from 10 episodes.
\begin{itemize}
	\item DelayedQ algorithm,
		$\epsilon_1 = 0.005$, $\delta = 0.1$, $\text{MaxR} = 1.0$, $m = 15$,
		300000 timesteps.
	\item Q-learning,
		learning rate decay from 0.05 to 0.01, 
		$\epsilon$-greedy exploration decay from 1.0 to 0.1.
		300000 timesteps.
	\item Q-learning with our Reward Shaping,
		learning rate decay from 0.05 to 0.01,
		$\epsilon$-greedy exploration decay from 1.0 to 0.1.
		290000 timesteps.
		The plot shows the performance of~$\Est{\Policy}_1^*$.
\end{itemize}

\subsubsection*{Figure~\refmain{fig:8rooms-inv}}
\paragraph{Environment}
Same environment as in Figure~\refmain{fig:8rooms-steps}.

\paragraph{Algorithms}
Each point in the plot shows average and standard deviation computed over evaluations of 10 different runs.
Each evaluation computes the average episode length from 10 episodes.
The plots show the performance of~$\Est{\Policy}_1^*$.
\begin{itemize}
	\item Q-learning with our Reward Shaping,
		learning rate decay from 0.05 to 0.01,
		$\epsilon$-greedy exploration decay from 1.0 to 0.1.
		290000 timesteps.
	\item Q-learning with the same potential we use in our Reward Shaping,
		with the only difference the potential is zero when the episode ends.
		Learning rate decay from 0.05 to 0.01,
		$\epsilon$-greedy exploration decay from 1.0 to 0.1.
		290000 timesteps.
\end{itemize}

\subsubsection*{Figure~\refmain{fig:reach-abs-errors}}
\paragraph{Environment}
Same environment as in Figure~\refmain{fig:8rooms-steps}.

\paragraph{Algorithms}
Each point in the plot shows average and standard deviation computed over evaluations of 10 different runs.
Each evaluation computes the average episode length from 10 episodes.
The three agents are trained with the same parameters:
\begin{itemize}
	\item Q-learning with our Reward Shaping,
		learning rate decay from 0.05 to 0.01,
		$\epsilon$-greedy exploration decay from 1.0 to 0.1.
		400000 timesteps.
		The plot shows the performance of~$\Est{\Policy}_1^*$.
\end{itemize}
The only difference between the three executions is in the potential induced by the abstract MDP $\Model_2$.
This is described in the main paper.

\subsubsection*{Figure~\refmain{fig:task-cont-dqn}}
\paragraph{Environment}

\begin{figure}
	\centering
	\includegraphics{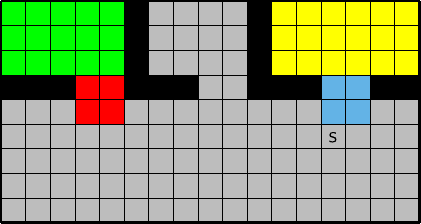}
	\caption{Map used for the plots in Figure~\refmain{fig:task-cont-dqn}.}
	\label{fig:office-map}
\end{figure}
In the most abstract environment, $\Model_{2,d}$,
each state represents a color of Figure~\refthis{fig:office-map}.
A transition is possible if two colors are adjacent in the map. 
Failure probability 10\% and discounting $\gamma_2 = 0.9$.
States of the lower model, $\Model_{1,d}$, represent the cells of Figure~\refthis{fig:office-map}. 
Failure probability 4\% and discounting $\gamma_1 = 0.98$.
In the ground dynamics, $\Model_{0,d}$, the agent moves with continuous increments over the same map.
Specifically, the actions are $\Actions_0 = \{ \Const{Left}, \Const{Right}, \Const{Accelerate}, \Const{Decelerate}, \Const{Interact}, \Const{NoOp} \}$.
$\Const{Left}$ and $\Const{Right}$ actions rotate the agent base by increments of 40°.
$\Const{Accelerate}$ and $\Const{Decelerate}$ increase and decrease linear velocity by 0.2 cells per step.
Maximum and minimum velocity are 0.6 and 0.0 cells per step.
The states are composed of $(x, y, \cos(\theta), \sin(\theta), v \cos(\theta), v \sin(\theta)) \in \States_0$.
Failure probability 3\%, episode maximum length 200, discounting $\gamma_0 = 0.99$.
Each model dynamics is $\Model_{i,d}$ is then composed with the automaton
of Figure~\refmain{fig:office-automa} to obtain a goal MDP~$\Model_i$.
Training is performed over each $\Model_i$. Further references are provided in the main paper.
The mapping $\Mapping_1$ preserves the current color and automaton state.
The mapping $\Mapping_0$ preserves the current cell and automaton state.
Note that, unlike in the 8-rooms map, the different regions touch in more than one location.

\paragraph{Algorithms}
Each point in the plot shows average and standard deviation computed over evaluations of 5 different runs.
Each evaluation computes the average episode length from 2 episodes.
\begin{itemize}
	\item Dueling DQN,
		learning rate 0.0005, 
		$\epsilon$-greedy exploration decay from 0.85 to 0.05,
		replay buffer of 80000 transitions,
		train batch size 64, 4 optimization per environment step,
		target network update frequency 1000 steps,
		840000 timesteps in total.
	\item Dueling DQN with our Reward Shaping.
		Same parameters as the previous run,
		700000 timesteps in total.
		The plot shows the performance of~$\BiasedStar{\Est{\Policy}}_0$.
\end{itemize}

Both agents used a Q-network composed of dense layers of sizes $[64, 64, 64\cdot\abs{\Automa}, 6]$,
interleaved with ReLu activations and batch normalization.
Observations are compositions of $\States_0$ and automaton states.
The first two dense layers process configurations $s \in \States_0$, only.
Their output features are then duplicated for each automaton state and processed independently in the third layer.
A final selection layer chooses which of the $\abs{\Automa}$ features is active, depending on the current automaton state, and it passes it to the output layer.
The Dueling DQN implementation is provided by Ray RLlib v1.13.0\footnote{\texttt{https://docs.ray.io/en/latest/rllib/index.html}}
and the network is implemented in Tensorflow~2.

\subsection{Hyper-parameter selection}

The environments described above are defined over simple dynamics.
Assuming that these are given, we focus on the algorithms hyper-parameters instead.
We tested with many configurations, mostly during a development phase of the software.
At the final software version, instead, according to the available computational budget, we performed a limited hyper-parameter grid search using Tune~\cite{liaw2018tune}.
On average, we tested from two to three values for most of the parameters described above.
We report some examples here:
\begin{description}
	\item[Delayed-Q]
		We have computed the theoretical hyper-parameters satisfying the PAC bounds reported in the original paper~\cite{strehl2006pac}.
		In particular, $m$, the number of samples between each attempted update was too high for the algorithm to be competitive with the other approaches.
		So, we experimented with much more practical values $[1000, 200, 50, 30, 15]$.
		A stable convergence was still guaranteed with $m = 15$,
		so we preferred this this lower value.
		The value of optimism of updates $\epsilon_1$ seemed to have mild impact with respect to~$m$. 
	\item[Q-learning]
		$\epsilon$-greedy exploration with linear decay is a common exploration strategy for Q-learning.
		We observed that 1.0 to 0.0 was a better performing decay with respect to lower values of $\epsilon$.
		Learning rates of $0.05$ showed a more stable training with respect to higher values $0.1, 0.2$,
		which appeared to be noisier in comparison.
	\item[Q-learning with our shaping]
		We have always used the same settings as for classic Q-learning.
\end{description}
A very important parameter is the total number of environment steps.
Since this directly determines sample efficiency and it is the property we wanted to optimize,
we started with longer trainings of about 300000 steps, for discrete domains, then decreasing it down to about 50000,
until it was still possible to observe convergence with one of the techniques.
\begin{description}
	\item[Dueling-DQN]
		We have experimented with networks of sizes $[64, 64, 64\cdot\abs{\Automa}]$,
		$[64, 64\cdot\abs{\Automa}]$, 
		$[128, 64\cdot\abs{\Automa}]$, 
		$[64, 64\cdot\abs{\Automa}, 64\cdot\abs{\Automa}]$.
		We tested learning rates of $[0.01, 0.005, 0.002, 0.001, 0.0005]$,
		replay buffer of sizes $50000, 80000, 100000$,
		target network updates between $200, 500, 1000, 2000$ steps.
		Among the limited search we could perform on the large product of possibilities,
		we settled on the parameters that showed a more stable convergence withing the
		total time steps budget.
		We used the same parameters for experimenting with and without shaping.
\end{description}

\subsection{Hardware information}
All experiments have been conducted on a Dell XPS 15, with an Intel Core i7-10750H CPU with 12 threads,
NVIDIA GeForce GTX 1650 Ti GPU, 15GB of RAM and 4GM of GPU memory.
Operating System Debian 11. We provide all the software information in Section~\refthis{sec:using}.

\section{Reproducing and using the software}
\label{sec:using}
\subsection{Description of the software}

The software is a Python~3 package called \texttt{multinav}\footnote{\texttt{https://github.com/cipollone/multinav2}}.
In input, it receives a configuration file in Yaml format, with all the environment and algorithm hyper-parameters.
In output, it generates checkpoints with the agent's parameters and log files with all the evaluation metrics.
This section explains the working principles of the software and the content of the attached archive.
For installation instructions, refer to the Installation and execution section.

We call \emph{run} a single execution of the software from a given seed.
With \emph{experiment} we refer to one or multiple runs executed from a single set of hyper-parameters
(an experiment generates a line in a plot).
In the archive attached to the present submission, there is a directory \texttt{experiments},
with the following substructure:
\begin{verbatim}
	<paper-figure>/<experiment-name>/<run-id>/logs/
	<paper-figure>/<experiment-name>/<run-id>/models/
\end{verbatim}
The \verb|<experiment-name>| is an identifier for an experiment.
For example, the Dueling DQN baseline of Figure~\refmain{fig:task-cont-dqn} is called 
\texttt{office2-0dqn}, while our RS approach is under \texttt{office2-0sh}.
\texttt{office2} is the name of an environment, while 0 refers to the fact that this
environment is the ground MDP~$\Model_0$.
\verb|<run-id>|, instead, is an identifier for the single run.
Usually, this is an integer from 0 to 9.

Under \texttt{logs} we find these important files:
\begin{description}
	\item[\texttt{experiment.yaml}] is the configuration file of an experiment.
		Inside it, there is the \verb!n-runs! parameter that controls how many runs are associated to this experiment.
		This file is not specifically needed for re-execution of single runs.
	\item[\texttt{run-options.yaml}] is the configuration file of a single run.
		This is the file that is actually passed to each execution of the \texttt{multinav} software.
		\texttt{run-options.yaml} contains all the algorithm and environment hyper-parameters that are in \texttt{experiment.yaml},
		plus some details for exact re-execution, such as the seed for this run and the commit hash of the software at the time of execution.
	\item[\texttt{algorithm-diff.patch}] In addition to the commit hash, which is stored in \texttt{run-options.yaml},
		we also save a patch file. This allow us to keep track, at the time of execution, of any source code changes since the latest commit.
\end{description}

The files in these directories can be used to inspect our work and re-execute trainings.
So, to re-execute run number 3 from our reward shaping agent, shown in Figure~\refmain{fig:8rooms-steps}, we do:
\begin{verbatim}
	cd multinav/
	cp -r ../experiments/fig2b/ outputs
	git checkout <commit-hash>
	git apply outputs/rooms8-1sh/3/logs/algorithm-diff.patch
	python -m multinav train --params outputs/rooms8-1sh/3/logs/run-options.yaml
\end{verbatim}
This will (over-)write the output files under the same directories.
It is possible to modify the keys \texttt{logs-dir} and \texttt{model-dir}
inside \texttt{run-options.yaml} to specify different output paths.

Copying the experiment directory under \texttt{outputs} is necessary,
because some agents might load data from other directories in \texttt{outputs}.
For example, any reward shaping agent that is learning on some $\Model_1$ needs to load the value function from $\Model_2$.
The git commands guarantee exact execution and needs to be issued at most once per experiment.
Finally, the last command starts the training from the same seed and configuration.

\subsection{Installation and execution}

\subsubsection*{Installation}
The software will be publicly released with a free software licence and it will be installable
from GitHub as \verb|pip install git+<url>|.
This will also install all the Python dependencies.
The only non-Python dependency is a tool called
``Lydia'' that we use for converting task specifications into reward functions
and Reward Machines (see below).

The instruction we provide here, instead, are for a full development installation,
because this procedure guarantees the exact dependencies we have used for
running the experiments.

\paragraph{Python version}
We use Python 3.9. If this version is already installed on the system, the user can skip this paragraph.
The 3.9 Python version can be installed with \texttt{pyenv}\footnote{\texttt{https://github.com/pyenv/pyenv}}.
Please refer to Pyenv installation instructions (note, in particular, the Python build dependencies).

\paragraph{Lydia}
Lydia is the only non-Python dependency we rely on.
Please refer to the package website\footnote{\texttt{https://github.com/whitemech/lydia}}
for installation instructions.
Simply, this reduces to executing from a Docker image:
\begin{verbatim}
	docker pull whitemech/lydia:latest
	alias lydia='docker run --rm -v$(pwd):/home/default whitemech/lydia lydia "$@"'
\end{verbatim}
Note that bash needs single quotes \textquotesingle. They may be wrongly typeset by LaTex in the listing above.

\paragraph{Package dependencies}
Python dependencies are managed by Poetry\footnote{\texttt{https://python-poetry.org/docs/}},
that we use to pin and install dependences versions.
Please refer to the Python Poetry website for installation instructions.
With a working \texttt{poetry} installation, it is sufficient to run
\begin{verbatim}
	cd multinav/
	poetry env use 3.9
	poetry install
\end{verbatim}
to install all the Python dependencies at the same version we have used for development.
It is advised to do so to avoid dependency issues.

\paragraph{Rllib patch}
When working with RLlib\footnote{\texttt{https://docs.ray.io/en/latest/rllib/index.html}} version 1.13.0,
we have encountered a bug in the DQN implementation. 
We provide a patch file to apply the necessary changes.
\begin{verbatim}
	ENV_PATH=$(poetry env info --path)
	RAY_PATH=$ENV_PATH/lib/python3.9/site-packages/ray
	cp dqn-fix.patch $(ENV_PATH)/
	cd $(ENV_PATH)
	patch -i dqn-fix.patch rllib/agents/dqn/dqn.py
\end{verbatim}
Make sure to execute these commands after \texttt{poetry install} as new installations
might revert these changes.

\paragraph{GPU}
The instructions above work for both CPU and GPU execution.
For CPU execution, the user should edit the \texttt{run-options.yaml} files
to set the \texttt{num\_gpus} option to 0.
If an NVIDIA GPU is present in the system. The user only needs to have the
appropriate drivers and CUDA runtime installed.

\subsubsection*{Execution}

Now that the dependency are installed, we can enter the Python environment containing
all the dependencies with
\begin{verbatim}
	cd multinav/
	poetry shell
\end{verbatim}
We assume all the following commands are executed inside this environment.

The program receives command line options. We can receive some help with
\begin{verbatim}
	python -m multinav --help
\end{verbatim}
Most of the parameters are provided through configuration files in Yaml format.
We provide many examples of \texttt{run-options.yaml} files in the attached archive.

To start a training we execute:
\begin{verbatim}
	python -m multinav train --params run-options.yaml
\end{verbatim}
To load and execute an already trained agent, we run:
\begin{verbatim}
	python -m multinav test --params <run-options.yaml> --load <checkpoint> [--render]
\end{verbatim}
For example, for a specific run,
\begin{verbatim}
	python -m multinav test --params outputs/office2-1sh/0/logs/run-options.yaml \
		--load outputs/office2-1sh/0/models/model_200000.pickle
\end{verbatim}
Refer to the previous section for a correct use of run options file and directories.

In the source code, the ``active'' agent refers to $\Biased{\DLearner}_i$ of Algorithm~\refmain{alg:main}.
The output is a value function/policy $\BiasedStar{\Est{\Policy}}$ stored under \texttt{models/model\_<step>.pickle}.
The ``passive'' agent refers to $\DLearner_i$, whose estimated optimum $\Est{\Policy}^*$ is stored under \texttt{models/Extra\_<step>.pickle}.